\documentclass{article}

\usepackage{natbib}
\setcitestyle{authoryear,open={(},close={)},aysep={,}}

\usepackage[preprint]{arXiv}

\usepackage{algorithm}
\usepackage{algorithmic}

\usepackage[utf8]{inputenc} 
\usepackage[T1]{fontenc}    
\usepackage{url}            
\usepackage{booktabs}       
\usepackage{amsfonts}       
\usepackage{nicefrac}       
\usepackage{microtype}      
\usepackage{xcolor}         
\usepackage{siunitx}
\usepackage{multirow}
\usepackage{makecell}

\usepackage{graphicx}
\usepackage{amsmath}
\usepackage{amsthm}
\usepackage{appendix}

\usepackage{tikz}
\usetikzlibrary{positioning, fit, backgrounds, arrows.meta, calc}

\newtheorem{theorem}{Theorem}[section]

\newtheorem{lemma}[theorem]{Lemma}
\newtheorem{definition}[theorem]{Definition}

\newcommand{\R}[0] {\mathbb R}

\author{%
  Hanna Mazzawi$^*$ \\
  Google Research \\
  \texttt{mazzawi@google.com} \\
  \And
  Benoit Dherin$^*$ \\
  Google Research\\
  \texttt{dherin@google.com} \\
  \And
  Michael Munn$^*$ \\
  Google Research \\
  \texttt{munn@google.com} \\
  \AND
  Adrian Goldwaser \\
  Google Research \& Cambridge University \\
  \texttt{goldwaser@google.com} \\  
  \AND
  Michael Wunder \\
  Google Research \\
  \texttt{mwunder@google.com} \\
  \And
  Javier Gonzalvo \\
  Google Research \\
  \texttt{xavigonzalvo@google.com} \\
}

\title{Transmuting prompts into weights}

\begin{document}

\maketitle

\def\thefootnote{*}\footnotetext{These authors contributed equally to this work}

\begin{abstract}
A growing body of research has demonstrated that the behavior of large language models can be effectively controlled at inference time by directly modifying their internal states, either through vector additions to their activations or through updates to their weight matrices. These techniques, while powerful, are often guided by empirical heuristics, such as deriving ``steering vectors'' from the average activations of contrastive prompts. Building on the foundational work of Dherin et al. (2025), who discovered that a prompt's influence mathematically maps to \emph{token-dependent} implicit weight updates and introduced the initial concept of a static thought patch for prompt compression, we elevate this framework into a robust algorithm for direct model editing. We derive a principled method for condensing this transient information into \emph{token-independent} thought vectors and thought matrices. These constructs provide a theoretical explanation for existing vector-and-matrix-based model editing techniques and offer a direct, computationally-grounded method for transmuting textual input into reusable weight updates for complex architectures and new knowledge injection.
\end{abstract}

\section{Introduction}

Recent advancements in controlling large language models (LLMs) at inference time generally fall into two empirical families: \emph{activation steering} via ``steering vectors'' added to hidden states \citep{Subramani2022ExtractingLS, turner2025steering}, and \emph{model editing} via targeted, low-rank modifications to weight matrices \citep{meng2022locating, mitchell2022fast}. While these intervention strategies are remarkably effective, their development has been largely heuristic, lacking a unified theoretical justification rooted in the transformer architecture. 

This work provides a theoretical framework answering why interventions like averaging contrastive activations or applying low-rank matrix updates succeed. Our research builds directly on the foundational theory introduced by \citet{dherin2025learning}. In their work, they proved that for a standard transformer block, the computational effect of an input prompt can be perfectly replicated by applying specific, \emph{token-dependent} rank-one updates—termed ``minimal token-patches''—to the block's feed-forward weights. While they utilized these implicit updates primarily as an analytical tool to uncover the internal learning dynamics of transformers, they also introduced the preliminary concept of a static ``thought patch''—a least-squares aggregation of these token-specific updates—as a proof-of-concept for prompt compression on simple linear regression tasks.

In this paper, we significantly expand upon this original idea. We elevate the thought patch from a theoretical  tool for prompt compression to a practical, scalable algorithm for direct model control and permanent knowledge injection. Our primary contribution is to demonstrate how these transient patches can be formally generalized into reusable, token-independent \textbf{thought vectors} $\delta(I)$ and \textbf{thought matrices} $\Delta(I)$ for modern, non-linear architectures (such as Gemma).

This framework establishes a direct connection between transformer computation and model control. We show that the heuristic practice of averaging contrastive activations is equivalent to the least-squares approximation of the optimal thought vector, and that low-rank model editing naturally emerges as a sum of rank-one updates. We empirically validate this by transmuting instructions into weights for algorithmic tasks and new knowledge injection, achieving performance comparable to full-context prompting.

\textbf{Related Work:} Our work bridges activation steering (e.g., representation engineering \citep{zou2023representation}, task vectors \citep{hendel2023incontext}) and model editing (e.g., ROME \citep{meng2022locating}, PME \citep{ruzzetti2025private}). While recent literature suggests a ``parametric continuum'' where information transitions from volatile prompts to rigid weights, previous methods rely on heuristics or architectural modifications \citep{chen2024exact}. We provide a first-principles strategy to condense generic prompt information into reusable weight updates for standard transformers adapting a theoretical strategy proposed in \cite{dherin2025learning} to modern transformers using techniques from \cite{goldwaser2025equivalence}. For an extended review of related literature, please refer to Appendix \ref{appendix:original_related_work}.

\section{Transmuting Prompts to Thought Patches}

\citet{dherin2025learning} proved that for a prompt $C = [I, x_1, \dots, x_n]$, the computational effect of the context chunk $I$ can be perfectly replicated on a shortened prompt $C \backslash I = [x_1, \dots, x_n]$ by applying \emph{token-dependent} updates to the transformer's last-layer bias and first-layer weight matrix:
\begin{eqnarray}
b_x(I)  & = &  \tilde b + \delta_x(I) \\
W_x(I) & = &  W + \Delta_x(I)
\end{eqnarray}
where $\delta_x(I) = A(C,x) - A(C\backslash I, x)$ and $\Delta_x(I) = \frac{W\delta_x(I) a_x^T}{\|a_x\|^2}$ (with $a_x$ being the attention output without context $I$). 

Because these patches must be recomputed for every token, they are impractical for durable storage. We solve this by finding a single, token-independent \emph{thought patch} $(\delta(I), \Delta(I))$ that minimizes the error against the true token patches across a collection of representative examples.

\textbf{The Thought Vector:} We find the optimal thought vector by minimizing the squared error against individual token vectors $(\delta_i)$. The solution is the mean:
\begin{equation}
\delta(I) := \frac 1n \sum_{i=1}^n \delta_i
= \frac 1n \sum_{i=1}^n \big( A(C,x_i) - A(C\backslash I, x_i)\big)
\end{equation}
This justifies the heuristic of averaging contrastive activations used in standard activation engineering.

\textbf{The Thought Matrix:} We approximate the thought matrix by minimizing the error of the token patch equation across all examples. 

\begin{theorem} \label{theorem:minimization_main}
Consider $n$ vectors $a_1, \dots, a_n$ in $\mathbb R^d$ forming operators $\Delta_i = \frac { W\delta_i a_i^T}{\|a_i\|^2}$. The minimization problem $\Delta(I) =\textrm{argmin}_M \sum_{i=1}^n \|Ma_i -  W\delta_i\|^2$ has a unique solution if and only if $Z = \sum_{i=1}^n a_i a_i^T$ is invertible. The global minimum is reached by:
\begin{equation}
\Delta(I) = \left(\sum_{i=1}^n W\delta_i a_i^T\right)Z^{-1}
\end{equation}
\end{theorem}

(Proof deferred to Appendix \ref{appendix:thought_matrix}). This result proves that the optimal thought matrix naturally emerges as a sum of rank-one updates (since each $\delta_i a_i^T$ is an outer product).

\section{The Thought Patching Algorithm}

To transmute a specific instruction $I$ into model weights, we employ an iterative layer-wise alignment procedure detailed in Algorithm \ref{algorithm:gemma_thought_matrix}. The core intuition is to align the internal representations of the model processing a raw input $x$ with those of the model processing the instructed input $[I, x]$. 

For a dataset $\mathcal{D}$ of task completions compatible with the instruction $I$, we perform two forward passes per example:
\begin{enumerate}
    \item \textbf{Contextual Pass (Target):} We run the vanilla model on the full prompt $[I, x_1, \dots, x_n]$. This generates the ``ground truth'' attention outputs $A_l$ at every layer $l$.
    \item \textbf{Non-Contextual Pass (Approximation):} We run the model on the input $[x_1, \dots, x_n]$ \emph{without} the instruction $I$, using the \emph{currently patched} weights.
\end{enumerate}

For each layer $l$, we compute the discrepancy $dz_l = A_l - a_l$. The thought vector update $db_l$ is computed as the mean of these discrepancies. The thought matrix update $dW_l$ is obtained by solving the least-squares minimization problem defined in Theorem \ref{theorem:minimization_main}.  In Appendix \ref{appendix:gemma_modifications}, we detail modification to Algorithm \ref{algorithm:gemma_thought_matrix} to adapt it to Gemma 3, following insights from \cite{goldwaser2025equivalence}. 

\subsection{Optimization Strategies}
To improve the stability of the updates and allow the algorithm to generalize across diverse prompts, we introduce three standard enhancements:

\textbf{Batching:} Rather than computing updates based on a single completion, we pool the attention values over a batch of completions $\mathcal{B} = \{[I, x^{(k)}]\}_{k=1}^{B}$ to reduce the variance of the thought patch estimation. The least-squares optimization is then performed on these pooled tensors simultaneously.

\textbf{Learning Rate:} Instead of applying the patch directly, we scale the update with a tunable step size $\eta$, gradually encoding the instruction into the weights: $W^{\textrm{patched}}_{l} \leftarrow W^{\textrm{patched}}_{l} + \eta \cdot dW_l$.

\textbf{Regularization:} To prevent the thought matrix from overfitting to the specific completions in a batch, we introduce $L_2$ regularization (Ridge Regression) with strength $\rho$ to the estimation step: $dW_l = \operatorname{argmin}_M \left\| M \mathbf{a}_l - W_l \mathbf{dz}_l \right\|_F^2 + \rho \|M\|_F^2$.

\begin{algorithm}[tb]
  \caption{Thought Patching}
  \label{algorithm:gemma_thought_matrix}
{
\begin{algorithmic}
  \STATE {\bfseries Input:} Completions $\mathcal{D}$ of $I$, Weights $W, b$, Batch size $B$, Regularization $\rho$, Learning Rate $\eta$
  \STATE {\bfseries Initialize:} $W^{\textrm{patched}} \leftarrow W$, $b^{\textrm{patched}} \leftarrow b$
  
  \STATE $//$ Iterate over dataset in batches
  \FOR{batch $\mathcal{B} = \{[I, \mathbf{x}^{(k)}]\}_{k=1}^B$ {\bfseries in} $\mathcal{D}$}
    
    \STATE $//$ \textbf{Contextual Pass:} Run orig model with $I$ for $k \in \mathcal{B}$
    \STATE $//$ Collect attentions $A$ for all layers across the batch
    \STATE $\mathbf{A} = \operatorname{Stack}(\{A_l(I, \mathbf{x}^{(k)}; \textrm{original})\}_{k,l})$
    
    \FOR{layer $l$ {\bfseries in} $1, \dots, L$}
      \STATE $//$ \textbf{Non-Contextual Pass:} Run model \textit{without} $I$
      \STATE $//$ Collect attentions $\mathbf{a}_l$ using current patched weights
      \STATE $\mathbf{a}_l = \operatorname{Stack}(\{A_l(\mathbf{x}^{(k)}; {\textrm{patched}})\}_{k})$
      
      \STATE $//$ Calculate discrepancies (residuals) for the batch
      \STATE $\mathbf{dz}_l = \mathbf{A}_l - \mathbf{a}_l$
      
      \STATE $//$ Compute Thought Vector over batch
      \STATE $db_l = \operatorname{Mean}(\mathbf{dz}_l)$
      
      \STATE $//$ Compute Thought Matrix via Least Squares
      \STATE $dW_l = \operatorname{RidgeSolver}(\mathbf{a}_l\rightarrow W_l^{\textrm{patched}} \mathbf{dz}_l, \text{reg}=\rho)$
      
      \STATE $//$ Apply scaled updates to the $l^{th}$ layer
      \STATE $W^{\textrm{patched}}_{l} \leftarrow W^{\textrm{patched}}_{l} + \eta \cdot dW_l$
      \STATE $b^{\textrm{patched}}_{l} \leftarrow b^{\textrm{patched}}_{l} + \eta \cdot db_l$
    \ENDFOR
  \ENDFOR
\end{algorithmic}
}
\end{algorithm}

\section{Experimental Validation}

We validated our approach using the Gemma 3.0 1B Instruction Tuned model \citep{gemma_2025}. Specific architectural modifications necessary for Gemma's Gated GeLU and RMSNorm layers, along with exact experimental prompts, are detailed in Appendix \ref{appendix:gemma_modifications} and \ref{appendix:experiment_prompts}.

\textbf{Arithmetic and Linguistic Tasks:} Table \ref{table:small_tasks} reports performance on fundamental algorithmic tasks. We observe that performance without the instruction, but with the thought patch applied, perfectly matches that of Gemma with the full prompt. In all cases, only 10 examples were required to achieve this parity.

\begin{table}[h!]
    \centering
    \caption{Thought patch performance on arithmetic instruction tasks.}
    \label{table:small_tasks}
    {
        \begin{tabular}{l|c|c}
            \hline
            \textbf{Task ($I$)}            &   \makecell{\textbf{Original model} \\ \textbf{w/ context } \\ (Accuracy $\%$)}   & \makecell{\textbf{Patched model} \\ \textbf{w/o context} \\ (Accuracy $\%$)}     \\
            \hline
            \hline
            \makecell{Multiply  numbers \\ e.g. 3, 4, 7 $\to$ 84}            & 100 $\pm$  0.0    & $100 \pm 0.0$  \\
            \hline
            \makecell{Sum  numbers \hspace{.2in} \\ e.g. 3, 4, 7 $\to$ 14}   & $100 \pm 0.0 $             & $100 \pm 0.0$  \\
            \hline
        \end{tabular}
    }
\end{table}

\textbf{Detoxification:} To test semantic steering robustness, we applied the method to sentence detoxification using the ParaDetox dataset \citep{logacheva-etal-2022-paradetox}. This requires transforming toxic input into a non-toxic paraphrase while preserving meaning. As shown in Table \ref{table:detoxify}, the patched model without context reduced toxicity from 90.53\% to 6.22\%, closely matching the full-context baseline (6.7\%) while preserving semantic BERT scores against the ground-truth detoxified statements.

\begin{table}[h!]
    \centering
    \caption{Thought patch performance on detoxification.}
    \label{table:detoxify}
    {
        \begin{tabular}{l|c|c}
        \hline
        \textbf{Method} & \textbf{Toxicity (\%)} $\downarrow$ & \textbf{BERT} $\uparrow$ \\
        \hline
        Toxic Comment & $90.53 \pm 1.3$ & - \\
        Neutral Reference & $5.62 \pm 0.59$ & - \\
        \hline \hline
        \makecell{\hspace{-.15in}Original model\\with context} & $6.7 \pm 0.7$ & $68.09 \pm 0.51$ \\
        \makecell{\hspace{-.15in}Patched model\\ without context} & $6.22 \pm 1.6$ & $62.16 \pm 0.25$ \\
        \hline
        \end{tabular}
    }
\end{table}

\textbf{Translation and General Knowledge:} We further assessed thought patches on translation (French/Spanish to English) and stored knowledge retrieval (Country to Capital). Table \ref{table:tasks} demonstrates that the modified model without instructions achieves performance highly comparable to full prompting, and significantly better than the unpatched baseline.

\begin{table}[h!]
    \centering
    \caption{Thought patch performance on linguistic and knowledge tasks.}
    \label{table:tasks}
    {
        \begin{tabular}{c|c|c}
        \hline
        \textbf{Task  ($I$)}  & \makecell{\textbf{Original model} \\ \textbf{w/ context}\\ (Accuracy $\%$)} & \makecell{\textbf{Patched model} \\ \textbf{w/o context}\\ (Accuracy $\%$)} \\
        \hline 
        \hline
        \makecell{French to English \\ e.g. \textit{feu} $\to$ \textit{fire}} &  $97.37 \pm 0.01$  & $92.74 \pm 0.25$ \\
        \hline
        \makecell{Spanish to English \\ e.g. \textit{fuego} $\to$ \textit{fire}} &  $98.4 \pm 0.01$ & $94.99 \pm 0.43$ \\        
        \hline
        \makecell{Country to Capital \\ e.g. France $\to$ Paris} &  $79.0 \pm 0.21$  & $47.6 \pm 1.02$ \\
        \hline
        \end{tabular}
    }
\end{table}

\textbf{New Knowledge Integration:} Finally, we evaluated the ability to encode entirely new, synthetic key-value dictionaries (Table \ref{table:dictionary}).

\begin{table}[h!]
    \centering
    \caption{Thought patch performance on new knowledge injection.}
    \label{table:dictionary}
    {
    \begin{tabular}{l|c|r}
            \hline
            \makecell{\textbf{Dictionary} \\ \textbf{Size ($I$)}}    &  \makecell{\textbf{Original model} \\ \textbf{w/ context} \\ (Accuracy $\%$)}          & \makecell{\textbf{Patched model} \\ \textbf{w/o context} \\ (Accuracy $\%$)} \\ 
            \hline
            \hline
               $ k = 1$  & $100 \pm 0.0$    & $100 \pm 0.0$ \\
                $k = 2$ &  $100 \pm 0.0$     &  $100 \pm 0.0$\\
                $k = 4$ &  $100 \pm 0.0$     &  $100 \pm 0.0$ \\
                $k = 8$ &  $100 \pm 0.0$     & $99 \pm 1.0$ \\
                $k = 16$ &  $100 \pm 0.0$    & $98 \pm 1.0$\\
                $k = 32$ &  $100 \pm 0.0$    &  $91 \pm 1.0$ \\
        \hline
    \end{tabular}
    }
\end{table}

The thought patches successfully absorbed new information into the MLP weights. However, performance degrades slightly as dictionary size increases, and degrades significantly if the retrieval templates differ from those used during patch computation (see Appendix \ref{appendix:experiment_prompts}). This highlights a limitation: the current least-squares formulation is susceptible to overfitting on specific prompt structures.Note that for this task, the thought vector is absorbed into the $W_{\textrm{down}}$ projection rather than the bias, as detailed in \ref{appendix:gemma_modif2}. Experiment details are in Appendix \ref{appendix:dictionary}.

\section{Conclusion}

We establish a theoretical link between the transient mechanics of transformer inference and permanent model modification. By aggregating ephemeral token patches into reusable thought vectors and thought matrices, we provide a principled, data-efficient method to durably encode prompt instructions into weights. This framework offers a unified mathematical explanation for the success of heuristic control methods, confirming that contrastive averaging and low-rank updates are native mechanisms of transformer instruction encoding. Future work will focus on advanced regularization to prevent structural overfitting during knowledge injection.

\clearpage
\newpage

\bibliographystyle{unsrtnat}

\clearpage
\newpage

\appendix

\section{Algorithm modification for Gemma}
\label{appendix:gemma_modifications}

This appendix details the specific adaptations required to apply the Thought Patching Algorithm \ref{algorithm:gemma_thought_matrix} to the Gemma 3 architecture \citep{gemma_2025}, which differs from the vanilla transformer block primarily in its use of RMSNorm and Gated GeLU feed-forward networks.

\subsection{Adapting the Thought Matrix for Gated GeLU and Pre-normalization}
\label{appendix:gemma_modif1}

A standard Gemma block is depicted in Figure \ref{fig:gemma_block}.

As established in Section 3, the thought vector $\delta(I)$ corresponds to a bias shift in the residual stream and can be computed from the output $A(C,x)$ of the conceptual layer. However, the computation of the thought matrix $\Delta(I)$ requires careful selection of the input activation vector $a_l$.

In a standard transformer, the MLP input is often the direct output of the attention mechanism (or the residual stream). In Gemma, the input to the Multi-Layer Perceptron (MLP) passes through a normalization layer (RMSNorm) first. Therefore, strictly applying the update to the matrices $W_{\textrm{up}}$ and $W_{\textrm{gate}}$ requires using the activations \emph{after} this normalization. We denote this activation as $\hat{a}_l = \operatorname{RMSNorm}(z_l)$, corresponding to the ``Normalized input'' point in the computational graph.

Furthermore, the Gemma architecture utilizes a Gated GeLU activation function~\cite{shazeer2020glu}, meaning the first layer of the MLP is split into two distinct projections: an upward projection $W_{\textrm{up}}$ and a gating projection $W_{\textrm{gate}}$. The activation is computed as:
\begin{equation}
    \operatorname{MLP}_{\text{in}}(\hat{a}_l) = (W_{\textrm{gate}} \hat{a}_l) \odot \sigma(W_{\textrm{up}} \hat{a}_l)
\end{equation}
where $\sigma$ is the activation function (e.g., GeLU) and $\odot$ denotes element-wise multiplication. Consequently, the single thought matrix optimization problem described in Theorem \ref{theorem:minimization_main} must be decoupled into two independent least-squares estimations similarly to~\citet{goldwaser2025equivalence}. 

For a given layer $l$, let $\mathbf{\hat{a}}_l$ be the stacked normalized activations and $\mathbf{dz}_l$ be the target residual discrepancies. We solve:
\begin{align}
    dW_{l,\textrm{up}} &= \operatorname{argmin}_M \left\| M \mathbf{\hat{a}}_l - W_{l,\textrm{up}}^{\text{patched}} \mathbf{dz}_l \right\|_F^2 + \rho \|M\|_F^2 \\
    dW_{l,\textrm{gate}} &= \operatorname{argmin}_M \left\| M \mathbf{\hat{a}}_l - W_{l,\textrm{gate}}^{\text{patched}} \mathbf{dz}_l \right\|_F^2 + \rho \|M\|_F^2    
\end{align}
This ensures that the thought patch correctly influences both the gating mechanism and the content projection.

\subsection{Absorbing the Thought Vector into Weights and Scale}
\label{appendix:gemma_modif2}

A significant architectural challenge in applying the thought patches to Gemma is the absence of explicit bias parameters in the MLP output layer. While the theoretical framework in Section 2 predicts a thought vector update $b^{\text{patched}} = b + \delta(I)$, the Gemma architecture computes the MLP output via a strictly linear projection $W_{l,\textrm{down}}$:
\begin{equation}
    x_{l,\text{out}} = x_{l,\text{in}} + W_{l,\textrm{down}} h_l
\end{equation}
where $h_l$ is the hidden state after the Gated GeLU activation for layer $l$. To implement the affine shift $\delta(I)$ without an architectural modification (i.e., without inserting a new bias tensor), we adopt the projection method proposed by \citet{goldwaser2025equivalence} to absorb the shift directly into $W_\text{down}$.

An additional complexity arises from the use of RMSNorm after the MLP in Gemma. To address this, we follow the method from~\citet{goldwaser2025equivalence}. The objective is to modify $W_{\text{down}}$ such that its output effectively includes the shift $\delta(I)$, accounting for the specific scaling characteristics of the layer.

Let $d_l$ and $d^C_l$ denote the output activations of the linear projection $W_{l,\text{down}}$ in the non-contextual and contextual passes, respectively. Let $s_l$ represent the RMSNorm scale parameter for layer $l$. We first compute a target activation vector, $d_l^{\text{goal}}$, which incorporates the desired thought vector shift adjusted by the layer's normalization scale:
\begin{equation}
    d_l^{\text{goal}} = (\delta(I) \oslash s_l) + d^C_l - d_l
\end{equation}
where $\oslash$ denotes element-wise division. This adjustment ensures that the bias shift aligns correctly with the layer's activation distribution.

To prevent the weight update from destabilizing the model by altering the norm of the MLP output significantly, we impose a magnitude constraint. We construct a normalized target $\hat{d}_l^{\text{goal}}$ that preserves the direction of the calculated goal but matches the magnitude of the ground-truth contextual activations:
\begin{equation}
    \hat{d}_l^{\text{goal}} = \frac{\|d_l^C\|_2}{\|d_l^{\text{goal}}\|_2} \cdot d_l^{\text{goal}}
\end{equation}

Finally, we find the optimal update for the projection matrix by solving the least-squares minimization problem over the activations $h_l$:
\begin{equation}
    dW_{l,\text{down}} = \operatorname{argmin}_{M} \left\| M h_l - \hat{d}_l^{\text{goal}} \right\|^2_F + \rho \|M\|_F^2
\end{equation}
Although this formulation technically allows for updating the RMSNorm scale $s_l$ to correct residual magnitude discrepancies, we find empirically that updating only the matrix $W_{\text{down}}$ and leaving $s_l$ as its original value is sufficient for high-fidelity instruction following and knowledge absorption and helps maintain performance.

\begin{figure*}[h]
    \centering
\begin{tikzpicture}[
    scale=0.7,
    node distance=0.7cm and 0.6cm,
    block/.style = {
        rectangle, 
        rounded corners, 
        draw=black, 
        fill=gray!20,
        minimum height=2.5em, 
        minimum width=5em, 
        text centered
    },
    oper/.style  = {
        circle, 
        draw=black, 
        font=\Large, 
        inner sep=0pt, 
        minimum size=7mm
    },
    conn/.style = {-{Stealth[length=3mm]}}, 
    annot/.style = {text=blue, font=\footnotesize, align=left, anchor=west}
]

    \node[block] (norm1) {$\text{RMSNorm}_3$};
    \node[block, above left=of norm1] (w_gate) {$W_{\text{gate}}$};
    \node[block, above=of norm1] (w_up) {$W_{\text{up}}$};
    \node[block, above=of w_gate] (gelu) {GeLU};
    \node[oper, at=(gelu -| w_up)] (mult1) {$\otimes$};
    \node[block, above=of mult1] (w_down) {$W_{\text{down}}$};
    \node[block, above=of w_down] (norm2) {$\text{RMSNorm}_4$};
    \node[oper, above=of norm2] (add) {$\oplus$};

    \coordinate[above=1cm of add] (output);
    \coordinate[below=0.5cm of norm1] (branch_point);
    \coordinate (res_corner) at ([xshift=1.2cm]w_up.east |- branch_point);

    \def\leftblockshift{9cm} 

    \node[block] (ctx_norm1) at ($(norm1)-(\leftblockshift,0)$) {$\text{RMSNorm}_1$};
    \node[oper] (ctx_add) at ($(add)-(\leftblockshift,0)$) {$\oplus$};

    \path (ctx_norm1) -- (ctx_add) 
        node[pos=0.33, block] (sa) {SelfAttention}
        node[pos=0.66, block] (ctx_norm2) {$\text{RMSNorm}_2$};

    \coordinate[below=1.3cm of ctx_norm1] (global_input);
    \coordinate[below=0.5cm of ctx_norm1] (ctx_branch_point);
    \coordinate (ctx_res_corner) at ([xshift=-1.2cm]sa.west |- ctx_branch_point);

    \draw[conn] (global_input) -- (ctx_branch_point) -- (ctx_norm1);
    \draw[conn] (ctx_norm1) -- (sa);
    \draw[conn] (sa) -- (ctx_norm2);
    \draw[conn] (ctx_norm2) -- (ctx_add);
    \draw[conn] (ctx_branch_point) -- (ctx_res_corner) |- (ctx_add);

    \coordinate (gap_center) at ($(sa.east)!0.5!(w_gate.west)$);
    
    \coordinate (path_top) at ($(ctx_add.north)+(0, 1.1)$);
    
    \coordinate (path_bottom) at ($(branch_point)-(0, 1.1)$);
    
    \draw[conn, rounded corners] (ctx_add) 
        -- (ctx_add |- path_top)       
        -- (gap_center |- path_top)    
        -- (gap_center |- path_bottom) 
        -- (branch_point |- path_bottom) 
        -- (norm1);             

    \draw[conn] (branch_point) -- (norm1);
    \coordinate (split) at ($(norm1.north)+(0,0.35)$);
    \draw (norm1) -- (split);
    \draw[conn] (split) -| (w_gate.south);
    \draw[conn] (split) -| (w_up.south);
    \draw[conn] (w_gate) -- (gelu);
    \draw[conn] (gelu) -- (mult1);
    \draw[conn] (w_up) -- (mult1);
    \draw[conn] (mult1) -- (w_down);
    \draw[conn] (w_down) -- (norm2); 
    \draw[conn] (norm2) -- (add); 
    \draw[conn] (add) -- (output);
    \draw[conn] (branch_point) -- (res_corner) |- (add);

    \def\lblx{5.5}
    \node[annot] (lbl_input) at ($(branch_point)+(\lblx,-0.2)$) {MLP Input:\\$\mathbf{a} = A(C,\mathbf{x})$};
    \draw[blue, ->] (lbl_input.west) -- ($(branch_point)+(0,-0.2)$);
    \node[annot] (lbl_norm) at ($(split)+(\lblx,-0.18)$) {Normalized input:\\$\hat{\mathbf{a}} = \text{RMSNorm}(A(C,\mathbf{x}))$};
    \draw[blue, ->] (lbl_norm.west) -- ($(split)+(0,-0.18)$);
    \coordinate (h_loc) at ($(mult1.north)!0.5!(w_down.south)$);
    \node[annot] (lbl_h) at ($(h_loc)+(\lblx,-0.1)$) {Hidden activations: $\mathbf{h}$};
    \draw[blue, ->] (lbl_h.west) -- ($(h_loc)+(0,-0.1)$);
    \coordinate (d_loc) at ($(w_down.north)!0.5!(norm2.south)$);
    \node[annot] (lbl_d) at ($(d_loc)+(\lblx,-0.1)$) {MLP output: $\mathbf{d}$};
    \draw[blue, ->] (lbl_d.west) -- ($(d_loc)+(0,-0.1)$);

    \begin{scope}[on background layer]
        \node[
            fill=gray!10, 
            draw=gray!70, 
            rounded corners,
            inner sep=10pt, 
            fit=(ctx_norm1) (sa) (ctx_norm2) (ctx_add) (ctx_res_corner),
            label={[gray!70, font=\bfseries, yshift=1cm]north:Contextual block: $A(C,\mathbf{x})$} 
        ] (ctx_background_box) {};

        \node[
            fill=gray!10, 
            draw=gray!70, 
            rounded corners,
            inner sep=10pt, 
            fit=(norm1) (w_up) (w_gate) (norm2) (add) (res_corner),
            label={[gray!70, font=\bfseries, yshift=1cm]north:MLP Block}
        ] (background_box) {};
    \end{scope}

\end{tikzpicture}
\caption{The Gemma block architecture. Note the pre-feedforward and post-feedforward normalization, the split gating mechanism and the lack of a bias parameter, which necessitate the specific algorithm modifications detailed in Appendix A.}
    \label{fig:gemma_block}
\end{figure*}

\clearpage
\newpage

\section{Experiment Details and Algorithm}
\label{appendix:experiments_details}

All experiments were conducted using the Gemma 3.0 1B model \citep{gemma_2025}. As noted in the main text, the specific architectural features of Gemma—namely RMSNorm and GeLU gating—necessitate the modifications detailed in Appendix \ref{appendix:gemma_modifications}.

For each experimental task, we record and align the model's intermediate activations across two distinct pass types:
\begin{enumerate}
    \item \textbf{Contextual Pass (Target):} The model processes the full prompt containing the specific instruction $I$ (the \texttt{context}).
    \item \textbf{Non-Contextual Pass (Approximation):} The model processes only the input data (\texttt{completion}) wrapped in the standard chat template, without the explicit instruction $I$.
\end{enumerate}

\subsection{Arithmetic and Linguistic Tasks}
\label{appendix:small_tasks}

\paragraph{Arithmetic Experiments.} 
The instruction context $I$ was defined as either ``Sum the numbers:'' or ``Multiply the numbers:''. To establish a rigorous baseline where the unpatched model achieves 100\% accuracy with context, we restricted the dataset to operations on three single-digit integers. 

For the non-contextual pass, the line ``Multiply the numbers:'' is omitted. We constructed the thought patches using a dataset $\mathcal{D}$ of 10 completions per task. We employed a learning rate of $\eta = 0.1$, a batch size of 1, and no regularization ($\rho = 0$). Convergence to 100\% accuracy on the training set was observed after 10 steps for both multiplication and summation. We validated the induced thought patches on a held-out test set of 20 randomly regenerated examples at each step, achieving 100\% peak test accuracy over the 10 steps for each of 5 random seeds.

\subsection{Detoxification}\label{appendix:detoxification} 

We used the ParaDetox dataset from HuggingFace. We used only 10 examples from the ParaDetox training split to compute the weight update with one training step and a learning rate $\eta = 1.0$ and no regularization $\rho = 0$. For evaluation we compared the original model with the given detoxification instruction against the patched model with no detoxification instruction; i.e., given only the toxic sentence. 

\subsection{Translation, Knowledge and Algorithmic Tasks}
\label{appendix:translation}

\paragraph{Translation Experiment.}
We evaluated the capability of thought patches to induce translation behavior from a foreign source language (e.g., French or Spanish) to English. We focus only on dictionary translation of specific words and not phrases. For the non-contextual pass, the instruction to translate is omitted and only the French (or Spanish) word is provided to the model. We computed the thought patches using a dataset of 10 completions using a learning rate of $\eta = 1.0$ and no regularization, $\rho = 0.0$. We validated the original model and the patched model on a hold out set of 100 French (or Spanish) words. Because there is some variability in translation, we evaluate performance based on semantic similarity using a pre-trained BERT model. This score is reported as a percentage.

\paragraph{Prior Knowledge.} For this task, the contextual prompt asks the model to return the capital city when provided a country or, vice versa, return the country when given a capital city. For these experiments we used 5 demonstration examples with a learning rate of 1.0 and no regularization $\rho = 0.0$. Since knowledge of a capital city or country is binary, we report the exact match accuracy. We compute the accuracy over 100 examples of pairs of country and capitals that were not used to compute the thought patch.

\paragraph{Algorithmic Task.} To assess simple algorithmic manipulation, we asked the model to determine the first letter of a string. We used 5 demonstrations, $\eta = 1.0$, and $\rho = 0.0$, reporting exact match accuracy over 100 held-out examples.

\subsection{New Knowledge Integration}
\label{appendix:dictionary}

We report retrieval accuracy for dictionary sizes $k \in \{1, 2, 4, 8, 16, 32\}$. For larger dictionaries ($k \geq 16$), we had to perform a rank reduction by keeping the 40 principal components of the matrices returned by the least-square solver. We observe also a significant loss of performance when the retrieval queries used to evaluate the patch are different from the ones used to compute the patch, which points to a need to properly tune the regularization parameter, and constitute a limitation of this work.

\begin{table}[h!]
    \centering
    \caption{Thought patch performance on new knowledge evaluated on different queries.}
    \label{table:dictionary_different_queries}
    {
    \begin{tabular}{|l|c|r|}
            \hline
             Dict.      & Accuracy          & Accuracy     \\
             Size (I)   & patched model     & patched model  \\
                        & same queries (\%)      & different queries (\%)\\
            \hline
                k = 1  & 100 $\pm$ 0.0    & 100 $\pm$  0.0  \\
                k = 2  & 100 $\pm$ 0.0    & 100 $\pm$  0.0  \\
                k = 4  & 100 $\pm$ 0.0    & 83.3 $\pm$ 4.0  \\
                k = 8  &  99 $\pm$ 1.0    & 79.1 $\pm$  6.7 \\
                k = 16 & 98 $\pm$ 1.0     & 53.0 $\pm$ 6.6  \\
                k = 32 & 91 $\pm$ 1.0     & 42.7 $\pm$ 7.8  \\
        \hline
    \end{tabular}
    }
\end{table}

\clearpage
\newpage

\section{Prompt Experiment Details}
\label{appendix:experiment_prompts}

To ensure precise alignment of the hidden states during the least-squares estimation, we include the model's generated response (the \texttt{answer}) in both passes. This ensures the patch minimizes the divergence between the \textit{instructed} and \textit{uninstructed} processing of the exact same token sequence. The prompt templates are structured as follows:

\begin{verbatim}
# Contextual Prompt (Target)
<start_of_turn>user
{context}{completion}<end_of_turn>
<start_of_turn>model
{answer}<end_of_turn>

# Non-Contextual Prompt (Approximation)
<start_of_turn>user
{completion}<end_of_turn>
<start_of_turn>model
{answer}<end_of_turn>
\end{verbatim}

\subsection{Detoxification Instructions}
For detoxification experiments, we used the following contextual instruction:
\begin{quote}
\texttt{Please detoxify the following sentence. Return a single detoxified statement.
Do not provide any commentary or suggestions. Return a single detoxified statement that is a detoxified version of the original sentence but that still has the same semantic meaning as the original. Here is the original toxic sentence:}
\end{quote}

\subsection{Translation Instructions}
For French-to-English we use the context $I$ to be:
\begin{verbatim}
Translate the following word from French
to English. Return only a single English
word that is the translation of the
French word {foreign_source_word}.
\end{verbatim}
Similarly, for Spanish-to-English we use the context:
\begin{verbatim}
Translate the following word from 
Spanish to English. Return only a single 
English word that is the translation of
the Spanish word {foreign_source_word}.
\end{verbatim}

\subsection{Knowledge Retrieval Instructions}
For retrieving capitals, the instruction prompt used to compute the thought patch over demonstrations is:
\begin{quote}
\texttt{Give me the capital of the following country. Return only a single city that is the capital of the country}
\end{quote}
And for retrieving countries:
\begin{quote}
\texttt{Give me the country of the following capital city. Return only a single country that is the country of the capital}
\end{quote}

\subsection{Dictionary Formats and Divergent Queries}
For the dictionary encoding task, the \texttt{context} was the data of the dictionary assignment \texttt{M=\{...\}}, while the \texttt{completion} was one of the following query versions, chosen randomly:
\begin{verbatim}
Answer with one word, \
  - the value of {dict_name}["{key}"] is
  - {dict_name}["{key}"] equals
  - the value of {dict_name} \
    for key "{key}" is
\end{verbatim}

We observed a significant loss of performance when the retrieval queries used to evaluate the patch are different from the ones used to compute the patch. Table \ref{table:dictionary_different_queries} reports the accuracy when we use the following queries to evaluate rather than the ones mentioned above:
\begin{verbatim}
Answer with one word, \
  - what value corresponds \
    to key "{key}" in {dict_name}?
  - retrieve the content \
    of {dict_name}["{key}"].
  - what is stored in {dict_name} 
    under the key "{key}"?
\end{verbatim}

\begin{table}[h!]
    \centering
    \caption{Thought patch performance on new knowledge evaluated on different queries.}
    \label{table:dictionary_different_queries}
    {
    \begin{tabular}{|l|c|r|}
            \hline
             Dict.      & Accuracy          & Accuracy     \\
             Size (I)   & patched model     & patched model  \\
                        & same queries (\%)      & different queries (\%)\\
            \hline
                k = 1  & 100 $\pm$ 0.0    & 100 $\pm$  0.0  \\
                k = 2  & 100 $\pm$ 0.0    & 100 $\pm$  0.0  \\
                k = 4  & 100 $\pm$ 0.0    & 83.3 $\pm$ 4.0  \\
                k = 8  &  99 $\pm$ 1.0    & 79.1 $\pm$  6.7 \\
                k = 16 & 98 $\pm$ 1.0     & 53.0 $\pm$ 6.6  \\
                k = 32 & 91 $\pm$ 1.0     & 42.7 $\pm$ 7.8  \\
        \hline
    \end{tabular}
    }
\end{table}

\clearpage
\newpage

\section{Thought Matrix Estimation Theorem}
\label{appendix:thought_matrix}

In Section 2, we introduced the thought matrix representing the thought expressed in a chunk $I$ of a prompt as the matrix $\Delta(I)$ that that minimizes the errors $\|\Delta(I)a_i - \Delta_i a_i\|^2$ for all completions $[I, x_1, \dots, x_i]$ formed by all the partial prompts in our collection of prompts, and where 
$$\Delta_i = \frac{\delta_i a_i^T}{\|a_i\|^2}$$
is the token matrix for token $x_i$ with attention $a_i$ in the absence of $I$ and $\delta_i$ the corresponding thought vector.

\begin{theorem}
\label{theorem:minimization}
Consider $n$ vectors $y_1, \dots, y_n$ in $\mathbb R^d$ with which we form the operators $\Delta_i = \frac {\delta_i y_i^T}{\|y_i\|^2} $ where the $\delta_i\in \mathbb R^d$ are fixed vectors.  
Then the following minimization problem over the space of $d\times d$ matrices
\begin{equation}\label{equation:minimization}
\min_M \sum_{i=1}^n \|My_i - \Delta_i y_i\|^2,
\end{equation}
has a unique solution \emph{if and only if} the operator
$Z = \sum_{i=1}^n y_i y_i^T$
is invertible. In this case the minimum is reached by
\begin{equation}
M = \left(\sum_{i=1}^n \delta_i y_i^T\right)Z^{-1},
\end{equation}
which is a global minimum.
\end{theorem}

\begin{proof}
The minimum is achieved at a critical point of the error function
\begin{equation}
L(M) = \sum_{i=1}^n \|My_i - \Delta_i y_i\|^2,
\end{equation}
which is a point at which its gradient vanishes, i.e., $\nabla_M L(M) = 0$. Since the gradient of $L$ is given by
\begin{equation}
\nabla_M L(M) = 2 \sum_{i=1}^n (My_i - \Delta_i y_i)y_i^T.
\end{equation}
then $M$ is a critical point if and only if:
\begin{equation}
\sum_{i=1}^n (My_i - \Delta_i y_i)y_i^T = 0,
\end{equation}
which we can rewrite as
\begin{equation}
M \left(\sum_{i=1}^n y_i y_i^T\right) =  \sum_{i=1}^n \Delta_i y_i y_i^T.
\end{equation}
Now since for the operator $\Delta_i$, we have the following property
\begin{equation}
\Delta_i y_i y_i^T = \frac{\delta_i (y_i^T y_i) y_i^T}{\|y_i\|^2}= \delta_iy_i^T
\end{equation}
then $M$ is a critical point of $L$ if and only if
\begin{equation} 
M Z =  \sum_{i=1}^n \delta_i y_i^T.
\end{equation}
Now if $Z$ is invertible, we obtain that the critical exists and is unique given by
\begin{equation} \label{equation:minimization_solution}
M =  \left(\sum_{i=1}^n \delta_i y_i^T\right) Z^{-1}.
\end{equation}

To prove the converse, suppose by contradiction that $M$ exists and is unique but that $Z$ is not invertible. Since $Z$ is not invertible, its image $V = \operatorname{Image}(Z)$ is a strick subspace of $\mathbb R^d$. Consider an operator $A$ that is the identity on $V$ but move the orthogonal space $V^T$ around. Then $MAZ = MZ$ but $M\neq MA$ since the two operators have a different action on $V^T$. This means that for the new operator $M' = MA$, we also have that 
\begin{equation}
M' Z =  \sum_{i=1}^n \delta_i y_i^T,
\end{equation}
which means that it is a critical point of $L_M$ and $M'\neq M$, contradicting the uniqueness assumption. 

Observe at this point, that to fully complete the proof, we should ensure that M is a global minimum, not just a critical point. Under our assumption that $Z$ is invertible, we just showed that $L(M)$ has a single critical point. Since $L(M)$ is positive and goes to infinity as $M$ becomes large, then this single critical point can only be a minima. Since there is only one minima, it is a global minima.
\end{proof}

\subsection{Invertibility Conditions for $Z$}
\label{appendix:invertibility_conditions}

We now gives conditions on the vector $y_1, \dots, y_n$ for the matrix $Z = \sum_i y_iy_i^T$ in Theorem \ref{theorem:minimization} to be invertible. The proof of our statements are in Appendix \ref{appendix:low_rank_lematata}:

\begin{enumerate}

\item $Z$ is invertible \emph{if and only if} $y_1, \dots, y_n \in \mathbb R^d$ span the whole vector space (see Lemma \ref{lemma:invertibility_when_independence}); however, in this case the form of the inverse is difficult to compute explicitly.

\item When $y_1, \dots, y_n$ is a basis of the space (which implies that $n=d$), then the inverse takes the form $Z^{-1} = \sum_i \omega_i \omega_i^T$, where the vectors $\omega_i$ are the rows of $Y^{-1}$. (See Lemma \ref{lemma:invertibility_when_basis}.)

\item When $y_1, \dots, y_n$ is an  orthonormal basis ($n=d$) of the space $Z^{-1} = I$. (See Lemma \ref{lemma:invertibility_when_orgonal_basis}.)

\item When $y_1, \dots, y_n$ are vectors independently sampled from a spherical distribution, then for $n$ large enough $Z^{-1} = \frac 1{\sigma^2n} I$, where $\sigma^2$ is the distribution variance. (See Lemma \ref{lemma:invertibility_when_spherical_distribution}.)

\end{enumerate}

\subsection{Getting an approximation of $M$ when $Z$ is not invertible}
\label{appendix:non-invertible case}

We are seeking operators $M$ such that when applied to the $y_i$'s they give back the $\delta_i$'s as closely as possible. That is, we are looking to minimize the following linear regression problem:
\begin{equation}
\min_M \sum_i \|M y_i - \delta_i\|^2
\end{equation}
In fact, the proof of Theorem \ref{theorem:minimization} shows us that any matrix satisfying the following equation is an optimum:
\begin{equation}
M Z = \delta Y^T,
\end{equation}
where $Z = YY^T$. The solution is clear when $Z$ is invertible and given in Theorem \ref{theorem:minimization}, but in practice it may not be. In this case, we can always add a small diagonal matrix $\epsilon = \operatorname{diag}(\epsilon_1, \dots, \epsilon_d)$ to $Z$ to render it invertible: $Z_e = e + Z = e(1 + e^{-1}Z)$, whose inverse can be approximated as
\begin{equation}
Z_e^{-1} = (1 + e^{-1}Z)^{-1}e^{-1} \simeq (1 - e^{-1}Z)e^{-1} = c - c Z c,
\end{equation}
where in the last equation we set $c = \operatorname{diag}(c_1, \dots, c_d)$ to be the inverse of $\epsilon$ (i.e. $c_i = 1/\epsilon_i)$. Now, we can solve $MZ_e = \delta Y^T$ using this approximated inverse, yielding:
\begin{equation}
M = \delta Y^T (c - c Z^2) = \delta Y^T c - c\delta Y^T c.
\end{equation}
In particular, if we take $c$ to be a constant $\lambda$ time the identity, we obtain the following approximation:
\begin{equation}
M = \lambda \sum_{i=1}^n \delta_i y_i^T - \lambda^2 \sum_{i,j=1}^n 
\langle y_i, y_j\rangle \delta_i y_j^T,
\end{equation}
where we will understand $\lambda$ as a tunable hyper-parameter.

\

\section{Useful properties of low-rank operators}
\label{appendix:low_rank_lematata}

Consider linear a application $A:\R^d\rightarrow \R^d$ represented by a matrix
\begin{equation}
A  =  \sum_{i=1}^r v_i w_i^T, 
\end{equation}
where $v_i,w_i$ are column vectors in $\R^d$ with $r < d$. If we write $V$ to be the matrix whose columns are the $v_i's$ and $W$ the matrix whose columns are the $w_i$'s we can write $A$ in matrix notation as
\begin{equation}
A = V W^T.
\end{equation}

\subsection{Image, kernel, and rank}

\begin{lemma}
\label{lemma:low_rank_spans}
Consider a low-rank operator with matrix given by $A = \sum_{i=1}^r v_i w_i^T$ as above. Let us denote by  $V = \operatorname{span}\{v_i\}$ and $W = \operatorname{span}\{w_i\}$ the linear subspaces spanned by the vectors $v_i$ and  $w_i$ respectively for $i=1,\dots, r$. Then the rank of $A$ is bounded by $r$. More precisely, we have that
\begin{equation}
\operatorname{rank}(A) \leq \min \{\dim V, \dim W\}\leq r.
\end{equation}
and that 
\begin{equation}
\operatorname{image}(A)\subset V 
\quad\textrm{and}\quad
W^\perp \subset \operatorname{kernel(A)}.
\end{equation}
Moreover when the $v_i$'s and the $w_i$'s are independent then we have equality everywhere:
\begin{equation}
\operatorname{rank}(A) = r, \quad \operatorname{image}(A) = V,\quad    \operatorname{kernel(A)} = W^\perp.
\end{equation}
\end{lemma}
\begin{proof}
First of all, we trivially have that $\operatorname{image}(A) \subset V$ and $W^\perp \subset \operatorname{kernel}(A)$. 
Now, since $\operatorname{rank}(A) = \dim \operatorname{image}(A)$ by definition, we immediately have the $\operatorname{rank}(A) \leq \dim V$. 
Now consider the kernel-image relation, that is, that $\dim \operatorname{image}(A) + \dim \operatorname{kernel}(A) = d$, where $d$ is the dimension of the space. Combining this relation with $W^\perp \subset \operatorname{kernel}(A)$, we obtain the inequality
\begin{equation}
\operatorname{rank}(A) + \dim W^\perp \leq d.
\end{equation}
Now since $\dim W^\perp = d - \dim W$, we obtain from the inequality above that
\begin{equation}
\operatorname{rank}(A) + d - \dim W \leq d,
\end{equation}
which yields that $\operatorname{rank}(A) \leq \dim W$, proving the first part of the statement. 

Now, let us consider the case when the $v_i$'s and the $w_i$'s are independent. By independence we immediately obtain that $\dim V = \dim W = r$, and therefore we get that $\operatorname{rank}(A) = r$. Since $\operatorname{image}(A) \subset V$ and both spaces have the same dimension $r$, they must coincide: $\operatorname{image}(A) = V$

As for the kernel, to show that $\operatorname{kernel}(A) = W^\perp$, we only need to show the other inclusion direction: $\operatorname{kernel}(A) \subset W^T$. Let us take a vector $x$ in the kernel of $A$. Then we get that
\begin{equation}
0 = Av = \sum_{i=1}^r v_iw_i^Tx = \sum_{i=1}^r \langle w_i, x\rangle w_i.
\end{equation}
By independence of the $w_i$'s, we obtain that $\langle w_i, x\rangle = 0$ for $i=1,\dots, r$. This exactly means that $x$ is orthogonal to $W$, that is, $v\in W^\perp$. Thus $\operatorname{kernel}(A) = W^\perp$. 

\end{proof}

\subsection{Independence, span, and basis}

\begin{lemma}
\label{lemma:invertibility_when_independence}
Let $y_1,\dots, y_n\in \mathbb R^d$ and consider the linear map
\begin{equation}
Z = \sum_{i=1}^n y_i y_i^T.
\end{equation}
Then $Z$ is invertible if and only if the $y_i$'s span the vector space $\mathbb R^d$. 
\end{lemma}

\begin{proof}
Suppose that the $y_i$'s span the whole vector space $\mathbb R^d$. By Lemma \ref{lemma:low_rank_spans}, we have that $\operatorname{image}(Z) = \operatorname{span}(y_1, \dots, y_n) = \mathbb R^n$, which means that $Z$ is subjective. An operator $Z:\mathbb R^d\rightarrow \mathbb R^d$ defined on the same space can be subjective if and only if it is bijective, i.e., invertible.  

To prove the converse, suppose that $Z$ is invertible. We need to prove that any vector $v \in\mathbb R^d$ can be written as a linear combination of the $y_i$'s. Since $Z$ is invertible, let us denote by $a = Z^{-1}v$ the inverse image of $v$ by $Z$. Now we have that
\begin{equation}
v  = Za 
 =  \left(\sum_{i=1}^n y_iy_i^T\right) a 
 =  \sum_{i=1}^n \alpha_i y_i,
\end{equation}
where $\alpha_i = y_i^Ta$, which finishes proving the converse.

\end{proof}

\begin{lemma}
\label{lemma:invertibility_when_basis}
Let $y_1,\dots, y_d\in \mathbb R^d$ be a basis, and consider the linear map
\begin{equation}
Z = \sum_{i=1}^n y_i y_i^T.
\end{equation}
Then $Z$ is invertible with inverse given by
\begin{equation}
Z^{-1} = (Y^{-1})^T Y^{-1} = \sum_{j=1}^n \omega_j \omega_j^T,
\end{equation}
where $Y$ is the matrix with columns $y_1,\dots, y_n$, and the $\omega_j$
s are the columns of the matrix $(Y^{-1})^T$.
\end{lemma}

\begin{proof}
If $y_1, \dots, y_d$ is a basis, it means that these vectors are independent. This means that the derterminant of the matrix $Y$ with the $y_i$'s as columns is non zero: $\det Y \neq 0$. This means that it is invertible, and so is its transpose. Now, since $Z = YY^T$, it is easy to verify that
\begin{equation}
(Y^T)^{-1} Y^{-1} Z = Z(Y^T)^{-1} Y^{-1} = I.
\end{equation}
Now that we have established that $Z^{-1} = (Y^{-1})^T Y^{-1}$, setting $\Omega = (Y^{-1})^T$, we see that $Z^{-1} = \Omega \Omega^T$, concluding that $Z^{-1} = \sum_{i=1}^n \omega_i\omega_i^T$ with $\omega_i$ being the columns of $\Omega$, and hence the rows of $Y^{-1}$ by definition.
\end{proof}

\begin{lemma}
\label{lemma:invertibility_when_orgonal_basis}
Let $y_1,\dots, y_d\in \mathbb R^d$ be an orthonormal basis, and consider the linear map
\begin{equation}
Z = \sum_{i=1}^n y_i y_i^T.
\end{equation}
Then $Z$ and its inverse $Z^{-1}$ are both the identity matrix.
\end{lemma}

\begin{proof}
Since $Zy_i = \|y_i\|^2 y_i = y_i$ for all basis vectors, we see that $Z$ is the identity matrix in this basis. But if a matrix is the identity in one basis, it's the identity in any basis. 
\end{proof}

\subsection{Spherical random distribution}

\begin{lemma}\label{lemma:orthogonalitly_and_identity}
If a matrix $P$ is preserved by the group of orthogonal transformations, i.e., if $P = Q^T P Q$ for each orthogonal transformation $Q$, then the matrix is a multiple of the identity matrix, i.e.,  $P = c I$.
\end{lemma}
\begin{proof}
Let denote us by $P_{ij}$ the entries of the matrix $P$. By using suitable orthonormal transformations as well as the relation $P = Q^TPQ$ we will first show that the off-diagonal elements of $P_{ij}$ with $i\neq j$ must be zero. Then using a different orthonormal transformation, we will see that the diagonal elements need all to be one in order to satisfy $P = Q^TPQ$.

Let us start with the off-diagonal elements. Consider the orthonormal transformation $Q$ that sends the basis vector $e_i$ into the basis vector $-e_i$ and leaves all other basis vectors unchanged (reflection in the $e_i$ direction). In coordinates, $Q$ is the matrix with $Q_{ll} = 1$ if $l\neq i$, $Q_{ii} = -1$ and all other entries set to zero. In this case, $P = Q^TPQ$ implies that 
\begin{eqnarray}
P_{ij} 
& = & \sum_{u,v}(Q^T)_{iu} P_{uv} Q_{vj},\\
& = & \sum_{u,v}Q_{ui} P_{uv} Q_{vj}, \\
& = &  Q_{ii} P_{ij} Q_{jj}, \\
& = & -  P_{ij},
\end{eqnarray}
and hence $P_{ij} = 0$ since zero is the only number which at the same time positive and negative. We can repeat this argument for any off-diagonal element. 

Let us now take care of the diagonal elements knowing that the off-diagonal elements of $P$ are zero. This means that $P = \operatorname{diag}(c_1, \dots, c_d)$. Consider now the orthonormal operator $Q$ that permutes the basis vector $e_i$ with the basis vector $e_j$ and leaves all other basis vectors unchanged (rotation in the $ij$-plane). In coordinates, $Q$ is the matrix such that $Q_{ll}=1$ if $l\neq i,j$, $Q_{ij} = Q_{ji} = 1$ and zero otherwise. Now with this transformation $P = Q^TPQ$ implies in coordinates that
\begin{eqnarray}
P_{ii} 
& = & \sum_{u,v}(Q^T)_{iu} P_{uv} Q_{vi},\\
& = & \sum_{u,v}Q_{ui} P_{uv} Q_{vi}, \\
& = & Q_{ji} P_{jj} Q_{ij}, \\
& = &   P_{jj},
\end{eqnarray}
and in other words that $c_i = c_j$. Since we can repeat this argument for any pair of basis vectors, we obtain that $P = c I$ where $c = c_1 = \cdots = c_d$, which completes the proof.
\end{proof}

\begin{definition}
We say that a random variable $X$ has a \emph{spherical distributions} when it is invariant under the orthogonal group; this means that $X$ and its transformation by an element of the orthogonal group $QX$ are \emph{identically} distributed for any orthogonal transformation $Q$.
\end{definition}

\begin{lemma}
\label{lemma:invertibility_when_spherical_distribution}
Consider the rank 1 operator $T_y = y y^T$, where $y$ is a spherically distributed random variable on $\mathbb R^d$ with covariance matrix $\sigma^2 I$. Then we have that the expectation of $T_y$ is a multiple of the identity: i.e.,
\begin{equation}
\mathbb E(T_y) = \sigma^2 I.
\end{equation}
In particular, its empirical mean can approximate $\sigma^2 I$ arbitrarily close by increasing the number $n$ of samples $y_i$ of the random variable $y$:
\begin{equation}
\frac 1n \sum_{i=1}^n y_iy_i^T \simeq \sigma^2 I.
\end{equation}
\end{lemma}

\begin{proof}
Let $Q$ be an orthogonal transformation. Since the distribution of the random variable $y$ is spherical, it means that $y$ and $z = Qy$ are identically distributed. Since both variables have the same distribution, then their associated rank 1 projectors $T_y = y y^T$ and $T_z = z z^T$ are also identically distributed. This implies in particular that they have identical means:
\begin{equation}
T:= \mathbb E(T_y) = \mathbb E(T_z).
\end{equation}
On the other hand, if we compute the expectation of $T_z$ directly we obtain
\begin{eqnarray}
\mathbb E(T_z) 
& = & \mathbb E\left(z^T z\right) \\
& = &  \mathbb E\left( Q^T y^Ty Q\right) \\
& = & Q^T \mathbb E (T_y) Q,
\end{eqnarray}
which implies that $T = Q^T T Q$ for all orthogonal matrix $Q$ since $T = \mathbb E(T_y) = E(T_z)$. Using now Lemma \ref{lemma:orthogonalitly_and_identity}, we conclude that $T = c I$ for a constant $c$. To determine the constant $c$ above, we compute the trace of $T_y$ using the cyclical property of the trace (i.e. $\operatorname{trace}(AB) = \operatorname{trace}(BA)$):
\begin{equation}
    \operatorname{trace}(T_y) = 
    \operatorname{trace}(y^T y)  = 
    \operatorname{trace}(y y^T) = 
    \operatorname{trace}(\|y\|^2) = \|y\|^2.
\end{equation}
Now taking the expectation on both sides of the equation $\operatorname{trace}(T_y) = \|y\|^2$ and using that the expectation of the trace is the trace of the expectation, we obtain that 
\begin{eqnarray}
\mathbb E(\|y\|^2)
& = & \operatorname{trace}(I c)  \\
& = & c \operatorname{trace}(I)\\
& = & c d.
\end{eqnarray}
Hence, we have that $c = \mathbb{E}(\|y\|^2)/d$ where $d$ is the dimension of the space. Now, we can easily evaluate the expectation of $\|y\|^2$ using the expectation formula for a quadratic form: $\mathbb E(y^TAy) = E(y)^T A E(y) + \operatorname{trace}(\Sigma A)$, where $\Sigma$ is the covariance matrix for $y$. In our case, $\mathbb E(y) = 0$ and $\Sigma = \sigma^2 I$ since the distribution is spherical. Therefore
\begin{equation}
\mathbb E(\|y\|^2) = \operatorname{trace}(\sigma^2 I) = \sigma^2 d,
\end{equation}
which gives us $c = \sigma^2$.
\end{proof}

\section{Related Work}
\label{appendix:original_related_work}

Our work provides a unified theoretical framework that explains and connects two prominent families of empirical methods for controlling large language models at inference time: activation steering with vectors and direct model editing with matrices. Recent literature suggests these methods operate on a "parametric continuum," where information transitions from volatile prompts to transient activation vectors, and finally to rigid model weights.

\paragraph{Activation Steering with Vectors.} One of the most popular methods for guiding a model's behavior is \emph{activation steering}, which involves adding a "steering vector" to the residual stream activations within each transformer block \citep{Subramani2022ExtractingLS}. While these vectors can be learned, they are commonly derived using a simple and effective heuristic: computing the difference between the model's average activations on a set of ``positive'' and ``negative'' prompts \citep{rimsky2024steering}. This core idea of using contrastive or averaged activations has been refined and extended in various ways. For instance, some methods use linear probes on the space of contrastive activations to find steering vectors \citep{li2023inference}, while others extract them from the principal component of the contrastive embedding differences \citep{zou2023representation}. 

The most similar vectors to our proposed thought vectors are those computed by simple averaging over contrastive samples \citep{turner2025steering,chen2025persona}, a method shown to be highly reliable in recent benchmarks \citep{unifiedSteeringMethods}. This concept has also been generalized to capture entire tasks, with "function vectors" \citep{todd2024function} and "task vectors" derived from contrastive prompts \citep{hendel2023incontext}. Building on this, \citep{chen2025persona} introduced "persona vectors" to monitor and control high-level character traits, demonstrating that complex behavioral personas can be isolated as single directions. Similarly, \citet{Liu2024IncontextV} proposed "in-context vectors" to condense demonstrations into a latent steering vector. More recently, \citet{saglam2025learning} demonstrated that these task vectors are modality-independent, confirming they represent functional encodings rather than lexical patterns. Additionally, \citet{li2025implicit} introduced "implicit in-context learning", proving that demonstration examples can be compressed into a single "context vector" that shifts the model's behavior without processing tokens, validating the transmutation hypothesis.

While heuristic vector addition is effective, \citet{chalnev2024improving} utilizes Sparse Autoencoders (SAEs) to target specific features for "surgical" steering. This approach builds directly on the foundational work of \citet{bricken2023monosemanticity} and \citet{templeton2024scaling}, who demonstrated that SAEs can decompose dense activations into interpretable features which, when clamped, robustly steer model behavior (e.g., the "Golden Gate Claude" experiment). Despite their success, benchmarks show that the performance of vector-based steering methods is not always fully reliable \citep{tan2024analysing, Pres2024TowardsRE}, suggesting that vector addition alone may be an incomplete representation of an instruction's full effect \citep{Brumley2024ComparingBA, yang2025taskvectorsincontextlearning}.

\paragraph{Model Editing with Matrices.}
A parallel line of research focuses on \emph{model editing} through direct modification of a model's weight matrices. These techniques often target the feed-forward layers (FFNs), which have been hypothesized to function as key-value memories storing factual information \citep{geva2021transformerfeedforwardlayerskeyvalue}. \citet{zhang2025distributional} refine this view, identifying FFNs as the locus of static distributional associations, while attention layers handle dynamic reasoning.
Rather than adding to activations, methods like ROME \citep{meng2022locating} and MEND \citep{mitchell2022fast} apply low-rank updates to the weight matrices to permanently alter knowledge. The formal structure of these matrix edits bears a strong resemblance to our proposed thought matrices. This technique extends to safety controls \citep{wei2024assessing} and toxicity reduction \citep{uppaal2024model}. 

Recent work explicitly links steering concepts to weight editing. \citet{gurarieh2025precise} introduced PISCES, which uses SAEs to identify concept directions (similar to steering) and then ablates them directly in the FFN parameters. Similarly, \citet{ruzzetti2025private} developed Private Memorization Editing (PME), targeting FFNs to unlearn memorized PII, further confirming that semantic prompts map to specific physical locations in the weights. Finally, \citet{shafran2025decomposing} propose decomposing MLP activations to map these "transmuted" features back to specific neuron combinations.

\paragraph{Bridging Empirical Methods with Theory.}
A theoretical explanation for why these interventions work has been missing. To our knowledge, no prior work provides a first-principles strategy to condense information from a generic prompt into a reusable weight update for standard transformers. While \citet{chen2024exact} demonstrated exact conversion of ICL to weights for linearized attention, their approach requires architectural modifications.

Our work fills this gap by building on the insights of \citep{dherin2025learning}, which extends the "ICL as implicit gradient descent" framework of \citet{vonOswald2023TransformersLI}, \citet{dai2023why}, and \citet{akyrek2023what}. While initial work focused on linear regression, \citet{cheng2024transformers} proved that Transformers implement \textit{functional} gradient descent to learn non-linear functions.
We prove that the effect of a prompt on a standard transformer can be replicated by vector and matrix updates. This aligns with the "task vector" emergence studies of \citet{yang2025task}, but provides the explicit mapping mechanism. Our contribution is to show how these transient, token-specific updates can be aggregated into token-independent \emph{thought patches}, transmuting prompts into durable weights.

\end{document}